\documentclass{article}
\usepackage[accepted]{icml2016}
\usepackage{rbase}

\usepackage{times}
\usepackage{apalike}
\usepackage{natbib} 

\usepackage{amsmath}
\usepackage{amsthm}
\usepackage{amssymb}
\usepackage{stmaryrd}

\usepackage{tablefootnote}
\usepackage{hyperref}

\usepackage{graphicx}
\usepackage{xypic}
\usepackage[all,2cell,ps]{xy}
\usepackage{tikz}
\usepackage{epsfig}
\usepackage{enumerate}
\usepackage{stmaryrd}
\usepackage{bbm}
\usepackage{algorithm}
\usepackage{algorithmic}
\usepackage{color}
\usepackage{paralist}
\usepackage{multicol}
\usepackage{soul}
\usepackage{adjustbox}

\newtheorem{thm}{Theorem}
\newtheorem{prop}{Proposition}

\newtheorem{defn}{Definition}

\renewcommand{\df}[1]{\textbf{#1}}
\renewcommand{\x}{\mathbf{x}}
\renewcommand{\y}{\mathbf{y}}
\newcommand{\z}{\mathbf{z}}

\newcommand{\anonym}[1]{}

\newcommand{\kmls}[1]{\mathfrak{K}_{#1}}
\newcommand{\klg}{k_{\textrm{LG}}}
\newcommand{\kflg}{k_{\textrm{FLG}}}
\newcommand{\kmlg}{\mathfrak{K}}

\newcommand{\smallpull}{\vspace{-6pt}}
\newcommand{\bigpull}{\vspace{-12pt}}

\newcommand{\Gfrak}{\mathfrak{G}}

\begin{document}
\twocolumn[
\icmltitle {The Multiscale Laplacian Graph Kernel}     
\icmlauthor{Risi Kondor}{risi@cs.uchicago.edu}
\icmladdress{Department of Computer Science and Department of Statistics,
            University of Chicago}
\icmlauthor{Horace Pan}{hopan@uchicago.edu}
\icmladdress{Department of Computer Science,
            University of Chicago}

] 

\begin{abstract}
Many real world graphs, such as the graphs of molecules, exhibit structure at multiple different scales, 
but most existing kernels between graphs are either purely local or purely global in character.  
In contrast, by building a hierarchy of nested subgraphs, 
the Multiscale Laplacian Graph kernels (MLG kernels) that we define in this paper can account for structure 
at a range of different scales. 
At the heart of the MLG construction is another new graph kernel, 
called the Feature Space Laplacian Graph kernel (FLG kernel),  
which has the property that it can lift a base kernel defined on the vertices of two graphs 
to a kernel between the graphs. 
The MLG kernel applies such FLG kernels to subgraphs recursively. 
To make the MLG kernel computationally feasible, we also introduce a randomized projection 
procedure, similar to the Nystr\"om method, but for RKHS operators. 
\end{abstract}

\section{Introduction}\label{sec:intro}

There is a wide range of problems in applied machine learning from web data mining \cite{InoWasMot03} to protein 
function prediction \cite{BorOngSchVisetal05} where the input space is a space of graphs. 
A particularly important application domain is chemoinformatics, where the graphs capture the structure of molecules.  
In the pharamceutical industry, for example, machine learning algorithms are regularly used to 
screen candidate drug compounds for safety and efficacy against specific diseases \cite{Kubinyi03}.  

Because kernel methods neatly separate the issue of data representation from the statistical learning component, 
it is natural to formulate graph learning problems in the kernel paradigm. 
Starting with \cite{Gaertner02}, a number of different graph kernels have been appeared in the literature 
(for an overview, see \cite{VisBorKonSch10}). 
In general, a graph kernel \m{k(\Gcal_1,\Gcal_2)} must satisfy the following three requirements:
\begin{compactenum}[(a)]
\item The kernel should capture the right notion of similarity between \m{\Gcal_1} and \m{\Gcal_2}. 
For example, if \m{\Gcal_1} and \m{\Gcal_2} are social networks, then \m{k} might capture to what extent 
clustering structure, degree distribution, etc.\;match up between them. 
If, on the other hand, \m{\Gcal_1} and \m{\Gcal_2} are molecules, then we are probably more interested in what 
functional groups are present in both, and how they are arranged relative to each other.
\item The kernel is usually computed from the adjacency matrices \m{A_1} and \m{A_2} of the two graphs, but 
(unless the vertices are explicitly labeled), it must be invariant to their ordering. 
In other words, writing the kernel explicitly in terms of \m{A_1} and \m{A_2}, we must have  
\m{k(A_1,A_2)\<=k(A_1,P A_2 P^\top)} for any permutation matrix \m{P}. 
\item The kernel should be efficiently computable. 
The time complexity of many graph kernels  is \m{O(n^3)}, 
where \m{n} is the number of vertices of the larger of the two graphs. 
However, when dealing with large graphs, we might only be able to afford \m{O(n^2)} or even \m{O(n)} complexity. 
On the other hand, in chemoinformatics applications, \m{n} might 
only be on the order of a \m{100}, permitting the use of more expensive kernels.   
\end{compactenum} 
Of these three requirements, the second one (permutation invariance) has proved to be the central constraint  
around which much of the graph kernels literature is organized. 

In combinatorics, any function \m{\phi(A)} that is invariant to reordering the vertices 
(i.e., \m{\phi(PAP^\top)=\phi(A)} for any permutation matrix \m{P}) is called a graph invariant \cite{MikkonenGraphicValues}.  
The permutation invariance requirement effectively stipulates that graph kernels must be built out of 
graph invariants. 
In general, efficiently computable graph invariants offered by the mathematics literature tend 
to fall in one of two categories:
\begin{compactenum}[(a)]
\item Local invariants, which can often be reduced to simply counting some local properties, 
such as the number of triangles, squares, etc. that appear in \m{\Gcal} as subgraphs. 
\item Spectral invariants, which can be expressed as functions of the eigenvalues of 
the adjacency matrix or the graph Laplacian.
\end{compactenum}
Correspondingly, while different graph kernels are motivated in very different ways from 
random walks \cite{Gaertner02} through shortest paths \cite{BorKri05, FerEtal13}
to Fourier transforms on the symmetric group \cite{KonBor08}, 
ultimately most graph kernels also reduce to computing a function of the two graphs that is either purely local or 
purely spectral. 
For example all kernels based on the ``subgraph counting'' idea (e.g., \cite{ShervEtal09}) 
are local. 
On the other hand, most of the random walk based kernels are reducable to 
a spectral form involving the eigenvalues of either the two graphs individually, or their Kronecker 
product \cite{VisBorKonSch10} and therefore are really only sensitive to the large scale structure of graphs.

In practice, it would be desirable to have a kernel that is inbetween these two extremes, 
in the sense that it can take structure into account at \emph{multiple different scales}. 
A kernel between molecules, for 
example, must be sensitive to the overall large-scale shape of the graphs 
(whether they are more like a chain, a ring, a chain that branches, etc.), but also to what smaller structures 
(e.g., functional groups) are present in the graphs, and how they are related to the global structure 
(e.g., whether a particular functional group is towards the middle or one of the ends of the chain). 

For the most part, such a multiscale graph kernel has been missing from the literature. 
One notable exception is the Weisfeiler--Lehman kernel \cite{ShervEtal11}, 
which uses a combination of message passing and hashing to build summaries of the local neighborhood 
vertices at different scales. 
However, in practice, the message passing step is usually only iterated a relatively small number of 
times, so the Weisfeiler--Lehman kernel is still mostly local. 
Moreover, the hashing step is somewhat ad-hoc and does not give rise to well behaved, local summaries: 
perturbing the edges by a small amount leads to completely different hash features. 

In this paper we present a new graph kernel, the Multiscale Laplacian Graph Kernel (MLG kernel), 
which, we believe, is the first kernel in the literature that can truly compare structure in graphs  
simultaneously at multiple different scales. 
We begin by defining a simpler graph kernel, called the Feature Space Laplacian Graph Kernel (FLG kernel) 
that only operates at a single scale (Section \ref{sec: spectral}). 
The FLG kernel combines two sources of information: a partial labeling of the nodes in terms of vertex features, 
and topological information about the graph supplied by its Laplacian. 
An important property of the the FLG kernel is 
that it can work with vertex labels provided implicitly, in terms of a ``base kernel'' on the vertices. 
Crucially, this makes it possible to apply the FLG kernel recursively. 


The Multiscale Laplacian Graph Kernel (MLG kernel), which is the central object of the paper and is defined in 
Section \ref{sec: multiscale}, uses exactly this recursive property of the FLG kernel to build a hierarchy of subgraph  kernels 
that are not only sensitive to the topological relationships between individual vertices, but also between 
subgraphs of increasing sizes. Each kernel is defined in terms of the preceding kernel in the hierarchy.

Efficient computability is a major concern in our paper, and recursively defined kernels, especially on combinatorial 
data structures, can be very expensive. Therefore, in Section \ref{sec: lowrank} we describe a strategy 
based on a combination of linearizing each level of the kernel (relative to a given dataset) and a randomized 
low rank projection, that reduces every stage of the kernel computation to simple operations involving 
small matrices, leading to a very fast algorithm. Finally, section \ref{sec: experiments} presents experimental  
comparisons of our kernel with competing methods. 

\section{Laplacian Graph Kernels}\label{sec: spectral}

Let \m{\Gcal} be a weighted undirected graph with vertex set \m{V=\cbr{\sseq{v}{n}}} and edge set \m{E}. 
Recall that the graph Laplacian of \m{\Gcal} is an \m{n\times n} 
matrix \m{L^\Gcal}, with 
\[L^\Gcal_{i,j}=
\begin{cases}
-w_{i,j} &\Tif~~\cbr{v_i,v_j}\tin E\\
\sum_{j\,:\,\cbr{v_i,v_j}\in E}w_{i,j}& \Tif~~ i\<=j \\
0&\Totherwise, 
\end{cases}
\]
where \m{w_{i,j}} is the weight of edge \m{\cbr{v_i,v_j}}. 
The graph Laplacian is positive semi-definite, and 
in terms of the adjacency matrix \m{A} and the weighted degree matrix \m{D}, 
it can be expressed as \m{L\<=D\<-A}. 

Spectral graph theory tells us that the low eigenvalue eigenvectors of \m{L^\Gcal} 
(the ``low frequency modes'') are informative about the overall shape of \m{\Gcal}. 
One way of seeing this is to note that for any vector \m{\z\tin \RR^n}
\[\z^\top\! L^\Gcal\, \z=\sum_{\cbr{i,j}\in E} w_{i,j} (z_i-z_j)^2,\]
so the low eigenvalue eigenvectors are the smoothest functions on \m{\Gcal}, in the sense that they 
vary the least between adjacent vertices. 
An alternative interpretation emerges if we  
use \m{\Gcal} to construct a Gaussian graphical model (Markov Random Field or MRF) 
over \m{n} variables \m{\sseq{x}{n}} with 
clique potentials \m{\phi(x_i,x_j)=e^{-w_{i,j}(x_i-x_j)^2/2}}  
for each edge and \m{\psi(x_i)=e^{-\eta x_i^2/2}} for each vertex. 
The joint distribution of \m{\x=\br{\sseq{x}{n}}^\top} is then 
\begin{multline}\label{eq: gm}
p(\x)\propto \brBig{\prod_{\cbr{v_i,v_j}\in E} e^{-w_{i,j}(x_i\<-x_j)^2/2}}  
\brBig{\prod_{v_i\in V} e^{-\eta x_i^2/2}}=\\ 
e^{-\x^\top\!\nts  (L+\eta I)\ts \x/2},
\end{multline}
showing that the covariance matrix of \m{\x} is \m{(L^\Gcal\<+\eta I)^{-1}}. 
Note that the \m{\psi} factors were only added to ensure that the distribution is normalizable, and 
\m{\eta} is typically just a small constant ``regularizer'':  
\m{L^\Gcal} actually has a zero eigenvalue eigenvector (namely the constant vector \m{n^{-1/2}(1,1,\ldots,1)^\top}), 
so without adding \m{\eta I} we would not be able to invert it. In the following we will call 
\m{L^\Gcal\<+\eta I} the regularized Laplacian, and denote it simply by \m{L}. 


Both the above views suggest that 
if we want define a kernel between graphs that is sensitive to their overall shape, 
comparing the low eigenvalue eigenvectors of their Laplacians is a good place to start.  
Following the MRF route, 
given two graphs \m{\Gcal_1} and \m{\Gcal_2} of \m{n} vertices, we can define the kernel between them 
to be a kernel between the corresponding distributions \m{p_1\<=\Ncal(0,L_1^{-1})} and \m{p_2\<=\Ncal(0,L_2^{-1})}.    
Specifically, we use the Bhattacharyya kernel 
\begin{equation}\label{eq: bhatta}
k(p_1,p_2)=\int \sqrt{p_1(x)} \sqrt{p_2(x)}\, dx,
\end{equation}
because for Gaussian distributions it can be computed in closed form \cite{JebKon03}, giving 
\[k(p_1,p_2)=\fr{\absbig{\nts \br{\ovr{2}L_1\<+\ovr{2} L_2}^{-1}\nts }^{1/2}}
{\absbig{\nts L_1^{-1}\nts }^{1/4}\; \absbig{\nts L_2^{-1}\nts }^{1/4}}\;. \]
If some of the eigenvalues of \m{L_1^{-1}} or \m{L_2^{-1}} are zero or very close to zero, 
along certain directions in space the two distributions in \rf{eq: bhatta} 
become very flat, leading to vanishingly small kernel values (unless the ``flat'' directions of the 
two Gaussians are perfectly aligned).  
To remedy this problem, similarly to \cite{KonJeb03}, 
we ``soften'' (or regularize) the kernel by adding some small constant \m{\gamma} times the identity to \m{L_1^{-1}} and \m{L_2^{-1}}. 
This leads to what we call the Laplacian Graph Kernel. 

\begin{defn}
Let \m{\Gcal_1} and \m{\Gcal_2} be two graphs of \m{n} vertices with (regularized) 
Laplacians \m{L_1} and \m{L_2}, respectively.  
We define the \df{Laplacian graph kernel (LG kernel)} with parameter \m{\gamma} 
between \m{\Gcal_1} and \m{\Gcal_2} as \smallpull 
\begin{equation}\label{eq: LG}
\klg(\Gcal_1,\Gcal_2)=\fr{\absbig{\nts \br{\ovr{2}S_1^{-1}\<+\ovr{2} S_2^{-1}}^{-1}\nts }^{1/2}}
{\absN{\nts S_1\nts }^{1/4}\; \absN{\nts S_2\nts }^{1/4}}\;, 
\end{equation}
where \m{S_1\<=L_1^{-1}\nts \<+\gamma I} and \m{S_2\<=L_2^{-1}\nts \<+\gamma I}. 
\end{defn}

By virtue of \rf{eq: bhatta}, the LG kernel is guaranteed to be positive semi-definite,  
and because the value of the overlap integral \rf{eq: bhatta} is largely determined 
by the extent to which the subspaces spanned by the largest eigenvalue eigenvectors 
of \m{L_1^{-1}} and \m{L_2^{-1}} 
are aligned, it effectively captures similarity between the overall shapes of \m{\Gcal_1} and \m{\Gcal_2}. 
However, the LG kernel does suffer from three major limitations:
\begin{compactenum}[1.]
\item It assumes that both graphs have exactly the same number of vertices. 
\item It is only sensitive to the overall structure of the two graphs, and not to how the two graphs 
compare at more local scales.
\item It is not invariant to permuting the vertices. 
\end{compactenum}
Our goal for the rest of this paper is to overcome each of these limitations, while retaining the 
LG kernel's attractive spectral interpretation. 

\subsection{Feature space LG kernel}

In the probabilistic view of the LG kernel, every graph generates random vectors \m{\x=\br{\sseq{x}{n}}^\top}  
according to \rf{eq: gm}, and the kernel between two graphs is determined by  
comparing the corresponding distributions. 
The invariance problem arises because the ordering of the variables \m{\sseq{x}{n}} is arbitrary: 
even if \m{\Gcal_1} and \m{\Gcal_2} are topologically the same, \m{\klg(\Gcal_1,\Gcal_2)} might 
be low if their vertices happen to be numbered differently. 

One of the central ideas of this paper is to address this issue by transforming from the 
``vertex space variables''  
\m{\sseq{x}{n}} to ``feature space variables''   
\m{\sseq{y}{m}}, where \m{y_i=\sum_j t_{i,j}(x_j)}, 
and each \m{t_{i,j}} only depends on \m{j} through local and reordering invariant 
properties of vertex \m{v_j}. 
If we then compute an analogous kernel to the LG kernel, but now between the distributions of the 
\m{\y}'s rather than the \m{\x}'s, the resulting kernel will be permutation invariant.  

In the simplest case, each \m{t_{i,j}} is linear, i.e.,  
\m{t_{i,j}(x_j)=\phi_i(v_j)\cdot x_j}, where 
\m{(\phi_1,\ldots,\phi_m)} is a collection of \m{m}
local (and permutation invariant) vertex features. 
For example, \m{\phi_i(v_j)} may be the degree of vertex \m{v_j}, or the value of 
\m{h^\beta(v_j,v_j)}, where \m{h} is the 
diffusion kernel on \m{\Gcal} with length scale parameter \m{\beta} (c.f., \cite{Alexa2009}). 
In the chemoinformatics setting, the \m{\phi_i}'s might be some way of encoding what type of atom is 
located at vertex \m{v_j}. 

The linear transform of a multivariate normal random variable is multivariate normal. In particular, 
in our case, letting \m{U=(\phi_i(v_j))_{i,j}}, we have  
\m{\EE(\y)\<=0} and \m{\Cov(\y,\y)\<=U\ts \Cov(\x,\x)\ts U^\top\!\<=U L^{-1} U^\top},   
leading to the following 
kernel, which is the workhorse of the present paper. 

\begin{defn}\label{def: FLG}
Let \m{\Gcal_1} and \m{\Gcal_2} be two graphs with 
regularized Laplacians \m{L_1} 
and \m{L_2}, respectively, \m{\gamma\geq 0} 
a parameter, and \m{(\sseq{\phi}{m})} a collection of \m{m} local vertex features. 
Define the corresponding feature mapping matrices  
\[[U_1]_{i,j}=\phi_i(v_j)\qquad [U_2]_{i,j}=\phi_i(v'_j) \]
(where \m{v_j} is the \m{j}'th vertex of \m{\Gcal_1} and \m{v_j'} is the \m{j}'th vertex of \m{\Gcal_2}). 
The corresponding \df{Feature space Laplacian graph kernel (FLG kernel)} is 
\begin{equation}\label{eq: FLG}
\kflg(\Gcal_1,\Gcal_2)=\fr{\absbig{\nts \br{\ovr{2}S_1^{-1}\<+\ovr{2} S_2^{-1}}^{-1}\nts }^{1/2}}
{\absN{\nts S_1\nts }^{1/4}\; \absN{\nts S_2\nts }^{1/4}}\;, 
\end{equation}
where \m{S_1\<=U_1 L_1^{-1} U_1^\top \nts \<+\gamma I}\; and\; \m{S_2\<=U_2 L_2^{-1} U_2^\top \nts \<+\gamma I}. 
\end{defn}

Since the \m{\sseq{\phi}{m}} vertex features, by definition, are local and invariant to 
vertex renumbering, the FLG kernel is permutation invariant. 
Moreover, because the distributions \m{p_1} and \m{p_2} now live in the space of features rather than the 
space defined by the vertices, there is no problem with applying the kernel to two graphs 
with different numbers of vertices. 

Similarly to the LG kernel, the FLG kernel also captures information about the global shape of graphs. 
However, whereas, intuitively, the former encodes information such as ``\m{\Gcal} is an elongated graph 
with vertex number \m{i} towards one and and vertex number \m{j} at the other'', the FLG kernel can capture 
information more like ``\m{\Gcal} is elongated with low degree vertices at one end and high degree 
vertices at the other''. 
The major remaining shortcoming of the FLG kernel  
is that it cannot take into account structure at multiple different scales. 

\subsection{The ``kernelized'' LG kernel}\label{sec: kernelized}

The key to boosting \m{\kflg} to a multiscale kernel is that it itself can be ``kernelized'', i.e.,  
it can be computed from just the inner products between the feature vectors of the vertices 
(which we call the base kernel) without having to know the 
actual \m{\phi_i(v_j)} features values.  

\ignore{
\begin{defn}
Let \m{\Gcal_1} and \m{\Gcal_2} be two graphs with vertex sets \m{V_1=\cbrN{\sseq{v^{(1)}}{n_1}}} and 
\m{V_2=\cbrN{\sseq{v^{(2)}}{n_2}}}, respectively, and \m{(\sseq{\phi}{m})} a collection of 
local vertex features, as in Definition \ref{def: FLG}. 
We define the \df{base kernel}  
\[\kappa\colon (V_1\cup V_2)\times (V1_\cup V_2)\to\RR\]
\end{defn}
}

\begin{defn}
Given a collection \m{\V \phi=(\sseq{\phi}{m})^\top} of local vertex features, we define the corresponding 
\df{base kernel} \m{\kappa} between two vertices \m{v} and \m{v'} as the dot product of their feature vectors:\;    
\m{\kappa(v,v')\<=\V \phi(v)\cdot \V \phi(v')}.
\end{defn}

Note that in this definition \m{v} and \m{v'} may be two vertices of the same graph, or of two different graphs. 
We first show that, similarly to the Representer Theorem for other kernel methods \cite{SchSmo02}, 
to compute \m{\kflg(\Gcal_1,\Gcal_2)} 
one only needs to consider the subspace of \m{\RR^m} spanned by the feature vectors of their vertices. 

\begin{prop}\label{prop: subspace}
Let \m{\Gcal_1} and \m{\Gcal_2} be two graphs with vertex sets 
\m{V_1=\cbrN{v_{1}\ldots v_{n_1}}} and \m{V_2=\cbrN{v'_{1}\ldots v'_{n_2}}}, and let 
\m{\cbr{\sseq{\xi}{p}}} be an orthonormal basis for the subspace 
\vspace{-6pt}
\[W=\Tspan\cbrbig{\V\phi(v_{1}),\ldots,\V\phi(v_{n_1}),\V \phi(v'_{1}),\ldots,\V\phi(v'_{n_2})}.\vspace{-4pt}\]
Then, \rf{eq: FLG} can be rewritten as \smallpull
\begin{equation}\label{eq: subspace}
\kflg(\Gcal_1,\Gcal_2)=\fr{\absbig{\nts \brbig{\ovr{2}\wbar{S}_1^{-1}\<+\ovr{2} \wbar{S}_2^{-1}}^{-1}\nts }^{1/2}}
{\absN{\nts \wbar{S}_1\nts }^{1/4}\; \absN{\nts \wbar{S}_2\nts }^{1/4}}\;, 
\vspace{-6pt}\end{equation}
where \m{[\wbar{S}_1]_{i,j}\<=\xi_i^\top\! S_1 \xi_j} and \m{[\wbar{S}_2]_{i,j}\<=\xi_i^\top\! S_2 \xi_j}. 
In other words, \m{\wbar{S}_1} and \m{\wbar{S}_2} are the projections of \m{S_1} and \m{S_2} to \m{W}. 
\end{prop}
\begin{proof}
The proposition hinges on the fact that \rf{eq: FLG} is invariant to rotation. 
In particular, if we extend \m{\cbr{\sseq{\xi}{p}}} to an orthonormal basis \m{\cbr{\sseq{\xi}{m}}} 
for the whole of \m{\RR^m}, let \m{O=[\sseq{\xi}{m}]} (the change of basis matrix) and set 
\m{\tilde S_1\<=O^\top\! S_1 O}, and \m{\tilde S_2\<=O^\top\! S_2 O}, then \rf{eq: FLG} can equivalently be 
written as \smallpull 
\begin{equation}\label{eq: subspace1}
\kflg(\Gcal_1,\Gcal_2)=
\fr{\absbig{\nts \brbig{\ovr{2}\tilde S_1^{-1}\<+\ovr{2} \tilde S_2^{-1}}^{-1}\nts }^{1/2}}
{\absN{\nts \tilde S_1\nts }^{1/4}\; \absN{\nts \tilde S_2\nts }^{1/4}}\:. \smallpull 
\end{equation}
However, in the \m{\cbr{\sseq{\xi}{m}}} basis \m{\tilde S_1} and \m{\tilde S_2} take on a special form. 
Writing \m{S_1} in the outer product form \smallpull 
\[S_1=\sum_{a,b=1}^{n_1}\V \phi(v_{a})\ts [L_1^{-1}]_{a,b}\ts \V \phi(v_{b})^\top\! +\gamma I \smallpull\]
and considering that for \m{i\<>p},\; \m{\inp{\V \phi(v_{a}),\xi_i}\<=0} shows 
that \m{\tilde S_1} splits into a direct sum \m{\tilde S_1=\wbar{S}_1\oplus \widehat{S}_1} of two matrices: 
a \m{p\<\times p} matrix \m{\wbar{S}_1} whose \m{(i,j)} entry is \smallpull 
\begin{equation}\label{eq: subspace2}
\xi_i^\top\! S_1\xi_j= 
\!\!\!\sum_{a,b=1}^{n_1}\!\! \inp{\xi_i,\V \phi(v_{1,a})} [L_1^{-1}]_{a,b} 
\inpN{\V \phi(v_{1,b}),\xi_j}+\gamma \delta_{i,j}, 
\end{equation}
where \m{\delta_{i,j}} is the Kronecker delta; 
and an \m{(n\<-p)\<\times(n\<-p)} dimensional matrix \m{\widehat{S}_1\<=\gamma I_{n-p}} 
(where \m{I_{n-p}} denotes the \m{n\<-p} dimensional identity matrix). 
Naturally, \m{\tilde S_2} decomposes into \m{\wbar{S}_2\<\oplus \widehat{S}_2} in an analogous way. 

Recall that for any pair of square matrices \m{M_1} and \m{M_2},\; 
\m{\absN{M_1\<\oplus M_2}=\absN{\nts M_1\nts}\cdot \absN{\nts M_2\nts}} and 
\m{(M_1\oplus M_2)^{-1}=M_1^{-1}\oplus M_2^{-1}}. 
Applying this to \rf{eq: subspace1} then gives \smallpull 
\begin{multline*}
\kflg(\Gcal_1,\Gcal_2)=
\fr{\absbig{\nts \brbig{\brbig{\ovr{2}\wbar S{}_1^{-1}\!\<+\ovr{2} \wbar {S}{}_2^{-1}}\oplus\gamma^{-1} I_{n-p}}^{-1}\nts }^{1/2}}
{\absbig{\nts \wbar S_1\<\oplus \gamma I_{n-k}\nts }^{1/4}\; \absbig{\nts \wbar S_2\<\oplus \gamma I_{n-k}\nts }^{1/4}}\;=\\ 
\fr{\absbig{\brbig{\ovr{2}\wbar S{}_1^{-1}\!\<+\ovr{2}\wbar {S}{}_2^{-1}}^{-1}\!\oplus\gamma I_{n-p}}^{1/2}}
{\absbig{\nts \wbar S_1\<\oplus \gamma I_{n-k}\nts }^{1/4}\; \absbig{\nts \wbar S_2\<\oplus \gamma I_{n-k}\nts }^{1/4}}\;=\\ 
\fr{\gamma^{(n-p)/2}}{\gamma^{(n-p)/4}\; \gamma^{(n-p)/4}}
\fr{\absbig{\brbig{\ovr{2}\wbar S{}_1^{-1}\!\<+\ovr{2}\wbar {S}{}_2^{-1}}^{-1}}^{1/2}}
{\absbig{\nts \wbar S_1\nts }^{1/4}\;\;\absbig{\nts \wbar S_2 \nts }^{1/4}}\:.\bigpull 
\end{multline*} 
\end{proof}

Similarly to kernel PCA \cite{MikSchSmoMuletal99} or the Bhattacharyya kernel, 
the easiest way to get a basis for \m{W} as required by \rf{eq: subspace} is to 
compute the eigendecomposition of the joint Gram matrix of the vertices of the two graphs. 

\begin{prop}\label{prop: Smx}
Let \m{\Gcal_1} and \m{\Gcal} be as in Proposition \ref{prop: subspace}, 
\m{\wbar V=\cbrN{\sseq{\wbar v}{n_1+n_2}}} be the union of their vertex sets  
(where it is assumed that the first \m{n_1} vertices are \m{\cbr{\sseq{v}{n_1}}} and the 
second \m{n_2} vertices are \m{\cbr{\sseq{v'}{n_2}}}),   
and define the joint Gram matrix \m{K\tin\RR^{(n_1+n_2)\times(n_1+n_2)}} as   
\smallpull
\[K_{i,j}=\kappa(\wbar v_i,\wbar v_j)=\V \phi(\wbar v_i)^\top \V \phi (\wbar v_j).\]
Let \m{\sseq{\V u}{p}} be (a maximal orthonormal set of) the non-zero eigenvalue eigenvectors of \m{K} 
with corresponding eigenvalues 
\m{\sseq{\lambda}{p}}. 
Then the vectors 
\begin{equation}\label{eq: xi}
\xi_i=\ovr{\sqrt{\lambda_i}} \sum_{\ell=1}^{n_1+n_2} [\V u_i]_\ell\: \V \phi(\wbar v_\ell)
\end{equation}
form an orthonormal basis for \m{W}. 
Moreover, defining \m{Q=[\lambda_1^{1/2}\V u_1, \ldots, \lambda_p^{1/2} \V u_p]\tin\RR^{p\times p}} 
and setting \m{Q_1=Q_{1:n_1,\,:}} and \m{Q_2=Q_{n_1+1:n_2,\,:}} (the first \m{n_1} and remaining \m{n_2} 
rows of \m{Q}, respectively), the matrices \m{\wbar S_1} and \m{\wbar S_2} appearing in \rf{eq: subspace} 
can be computed as 
\begin{equation}\label{eq: Smat}
\wbar S_1=Q_1^\top L_1^{-1} Q_1 +\gamma I, \quad~~ \wbar S_2=Q_2^\top L_2^{-1} Q_2 +\gamma I.
\end{equation}
\end{prop} 

\begin{proof}
For \m{i\neq j}, \bigpull 
\begin{multline*}
\xi_i^\top \xi_j=
\ovr{\sqrt{\lambda_i\lambda_j}} \sum_{k=1}^{n_1+n_2}\sum_{\ell=1}^{n_1+n_2} [\V u_i]_k 
\underbrace{\V\phi(\wbar v_k)^\top\! \V \phi(\wbar v_\ell)}_{\kappa(\wbar v_k,\wbar v_\ell)}[\V u_j]_\ell=\\
(\lambda_i \lambda_j )^{-1/2}\,\V u_i^\top K \V u_j=(\lambda_j/\lambda_i)^{1/2} \V u_i^\top \V u_j=0, 
\end{multline*}\vspace{-12pt}\\  
while for \m{i\<=j}, 
\begin{mequation}
\xi_i^\top \xi_j=\lambda_i^{-1}\,\V u_i^\top K \V u_i=\V u_i^\top \V u_i=1,
\end{mequation}
showing that \m{\cbrN{\sseq{\xi}{p}}} is an orthonormal set.  
At the same time, \m{p\<=\textrm{rank}(K)\<=\dim(W)} and \m{\sseq{\xi}{p}\tin W}, proving that 
\m{\cbrN{\sseq{\xi}{p}}} is an orthonormal basis for \m{W}.  

To derive the form of \m{\wbar S_1}, simply plug \rf{eq: xi} into \rf{eq: subspace2}:\smallpull  
\begin{multline*}
\xi_i^\top\!  S_1 \xi_j=
\ovr{\sqrt{\lambda_i\lambda_j}} \sum_{k=1}^{n_1}\sum_{\ell=1}^{n_1} \sum_{a,b=1}^{n}[\V u_i]_k 
\underbrace{\V\phi(\wbar v_k)^\top\! \V \phi(\wbar v_a)}_{\kappa(\wbar v_k,\wbar v_a)}\cdot \\ 
\cdot  [L_1^{-1}]_{a,b} 
\underbrace{\V\phi(\wbar v_b)^\top\! \V \phi(\wbar v_\ell)}_{\kappa(\wbar v_b,\wbar v_\ell)} [\V u_j]_\ell+ 
\gamma \delta_{i,j}= \\ 
(\lambda_i \lambda_j)^{-1/2}\ts \V u_i^\top\nts K L^{-1} K \V u_j+\gamma \delta_{i,j}=\\
(\lambda_i \lambda_j)^{1/2}\ts \V u_i^\top\nts L^{-1} \V u_j+\gamma\delta_{i,j},  
\end{multline*}
and similarly for \m{\wbar S_2}. 
\end{proof}

As in other kernel methods, the significance of Propositions \ref{prop: subspace} and \ref{prop: Smx} 
is not just that they show show how  
\m{\kflg(\Gcal_1,\Gcal_2)} can be efficiently computed when \m{\V \phi} is very high dimensional, 
but that they also make it clear that the FLG kernel can really be induced from any base kernel,  
regardless of whether it corresponds to actual finite dimensional feature vectors or not. 
For completeness, we close this section with this generalized definition of the FLG kernel.  

\begin{defn}\label{def: gFLG}
Let \m{\Gcal_1} and \m{\Gcal_2} be two graphs. Assume that each of their vertices comes from an 
abstract vertex space \m{\Vcal} and that \m{\kappa\colon \Vcal \times\Vcal\to\RR} 
is a symmetric positive semi-definite kernel on \m{\Vcal}. 
The \df{generalized FLG kernel} induced from \m{\kappa} is then defined as \smallpull  
\begin{equation}\label{eq: gFLG}
\kflg^\kappa(\Gcal_1,\Gcal_2)=\fr{\absbig{\nts \brbig{\ovr{2}\wbar{S}_1^{-1}\<+\ovr{2} \wbar{S}_2^{-1}}^{-1}\nts }^{1/2}}
{\absN{\nts \wbar{S}_1\nts }^{1/4}\; \absN{\nts \wbar{S}_2\nts }^{1/4}}\;, 
\vspace{-6pt}\end{equation}
where \m{\wbar{S}_1} and \m{\wbar S_2} are as defined in Proposition \ref{prop: Smx}. 
\end{defn}

\ignore{
In particular, for a pair of graphs \m{(\Gcal_1,\Gcal_2)} with vertex sets 
\m{V_1=\cbrN{v_1\ldots v_{n_1}}} and 
\m{V_2=\cbrN{v'_1\ldots v'_{n_2}}}, we define the joint Gram matrix 
between their vertices as  
\[K_{i,j}=
\begin{cases} 
\kappa(v_i,v_j)&\Tif~~ i\<\leq n_1,~~j\<\leq n_1\\  
\kappa(v_i,v'_{j-n_1})&\Tif~~ i\<\leq n_1,~~j\<>n_1\\  
\kappa(v'_{i-n_1},v_{j})&\Tif~~ i\<>n_1,~~j\<\leq n_1\\  
\kappa(v'_{i-n_1},v'_{j-n_1})&\Tif~~ i\<>n_1,~j\<>n_1.  
\end{cases}\]
We let \m{K\<=Q^\top D Q} be the eigendecomposition of this matrix (so that the rows of \m{Q} are 
the eigenvectors) and define \m{Q_1\<=Q_{:,1:n_1}} and \m{Q_2\<=Q_{:,n_1+1:n_1+n_2}} 
(the first \m{n_1} columns of \m{Q} and the remaining \m{n_2} columns, respectively). 
We then have the following theorem. 

\begin{thm}\label{thm: flg compute}
Let \m{\Gcal_1} and \m{\Gcal_2} be two graphs and \m{\kappa} the base kernel between their vertices. 
Then, the FLG kernel \rf{eq: FLG} kernel can be equivalently written as 
\begin{equation*}
k(\Gcal_1,\Gcal_2)=\fr{\absbig{\nts \brbig{\ovr{2}\wbar{S}_1^{-1}\<+\ovr{2} \wbar{S}_2^{-1}}^{-1}\nts }^{1/2}}
{\absN{\nts \wbar{S}_1\nts }^{1/4}\; \absN{\nts \wbar{S}_2\nts }^{1/4}}\;, 
\end{equation*}
where \m{\wbar{S}_1\<=Q_1 L_1^{-1} Q_1^\top \nts \<+\gamma I}\: and\: 
\m{\wbar{S}_2\<=Q_2 L_2^{-1} Q_2^\top \nts \<+\gamma I}. 
\end{thm}

\begin{proof}
The theorem hinges on the fact that the FLG kernel is rotation invariant. This means that if 
\m{\cbr{\sseq{\xi}{m}}} is any orthonormal basis of \m{\RR^m} and \m{O} be is the corresponding  
change of basis matrix \m{O=[\sseq{\xi}{m}]}, then \rf{eq: FLG} remains invariant under  
changing \m{S_1} to \m{\tilde S_1=O^\top\! S_1 O}, and 
\m{S_2} to \m{\tilde S_2=O^\top\! S_2 O}. 

The specific orthonormal basis that we will use is adapted to the subspace \m{W} 
spanned by the feature vectors \m{\cbrN{\V\phi(v_1)\ldots\V\phi(v_{n_1}),\V\phi(v'_{1})\ldots\V\phi(v'_{n_2})}}:  
the first \m{p} basis vectors \m{\sseq{\xi}{p}} form a basis for \m{W}, 
while \m{\xi_{p+1},\ldots,\xi_m} form a basis for the complementary subspace \m{W^\perp}.  
By definition, \m{S_1} can be written in the outer product form 
\[S_1=\sum_{a,b=1}^{n_1} \V \phi(v_a) [L_1^{-1}]_{a,b} \V \phi(v_b)^\top +\gamma I.\]
However, for \m{i\<>p},\; \m{\inp{\V \phi(v_a),\xi_i}\<=0}. Therefore, \m{\tilde S_1} can be written as a direct sum 
\m{\tilde S_1=\wbar{S}_1\oplus \widehat{S}_1} of two smaller matrices: 
the \m{p\<\times p} matrix \m{\wbar{S}_1} with entries 
\[[\wbar{S}_1]_{i,j}=\sum_{a,b=1}^{n_1} \inp{\xi_i,\V \phi(v_a)} [L_1^{-1}]_{a,b} 
\inpN{\phi(v_b),\xi_j}+\gamma I_p,\]
and the \m{(n\<-p)\<\times(n\<-p)} dimensional matrix \m{\widehat{S}_1\<=\gamma I_{n-p}}. 
In both of these expressions \m{I_r} denotes the \m{r} dimensional identity matrix. 
Naturally, \m{\tilde S_2} decomposes into \m{\wbar{S}_2\<\oplus \widehat{S}_2} in an analogous way. 

Next, recall that for any pair of square matrices \m{M_1} and \m{M_2}, we have that 
\m{\absN{M_1\<\oplus M_2}=\absN{\nts M_1\nts}\cdot \absN{\nts M_2\nts}}. 
Furthermore, \m{(M_1\oplus M_2)^{-1}=M_1^{-1}\oplus M_2^{-1}}. 
In our particular case, 
\begin{multline*}
k(\Gcal_1,\Gcal_2)=
\fr{\absbig{\nts \brbig{\ovr{2}\tilde S_1^{-1}\<+\ovr{2} \tilde S_2^{-1}}^{-1}\nts }^{1/2}}
{\absbig{\nts \tilde S_1\nts }^{1/4}\; \absbig{\nts \tilde S_2\nts }^{1/4}}\;=\\ 
\fr{\absbig{\nts \brbig{\brbig{\ovr{2}\wbar S{}_1^{-1}\!\<+\ovr{2} \wbar {S}{}_2^{-1}}\oplus\gamma^{-1} I_{n-p}}^{-1}\nts }^{1/2}}
{\absbig{\nts \wbar S_1\<\oplus \gamma I_{n-k}\nts }^{1/4}\; \absbig{\nts \wbar S_2\<\oplus \gamma I_{n-k}\nts }^{1/4}}\;=\\ 
\fr{\absbig{\brbig{\ovr{2}\wbar S{}_1^{-1}\!\<+\ovr{2}\wbar {S}{}_2^{-1}}^{-1}\!\oplus\gamma I_{n-p}}^{1/2}}
{\absbig{\nts \wbar S_1\<\oplus \gamma I_{n-k}\nts }^{1/4}\; \absbig{\nts \wbar S_2\<\oplus \gamma I_{n-k}\nts }^{1/4}}\;=\\ 
\fr{\gamma^{(n-p)/2}\; \absbig{\brbig{\ovr{2}\wbar S{}_1^{-1}\!\<+\ovr{2}\wbar {S}{}_2^{-1}}^{-1}}^{1/2}}
{\gamma^{(n-p)/4}\; \absbig{\nts \wbar S_1\nts }^{1/4}\; \gamma^{(n-p)/2}\; \absbig{\nts \wbar S_2 \nts }^{1/4}}\;=\\ 
\fr{\absbig{\brbig{\ovr{2}\wbar S{}_1^{-1}\!\<+\ovr{2}\wbar {S}{}_2^{-1}}^{-1}}^{1/2}}
{\absbig{\nts \wbar S_1\nts }^{1/4}\; \absbig{\nts \wbar S_2 \nts }^{1/4}}.
\end{multline*}

\end{proof}
Theorem \rf{thm: flg compute} is analogous to similar results for kernel PCA \cite{MikSchSmoMuletal99} 
and the Bhattacharyya kernel in that it shows that while \m{m} might be very large, the FLG kernel 
can be computed by only considering an at most \m{n_1\<+n_2} dimensional subspace of \m{\RR^{m}}, 
spanned by \m{\cbrN{\V\phi(v_1)\ldots\V\phi(v_{n_1}),\V\phi(v'_{1})\ldots\V\phi(v'_{n_2})}}. 
The eigendecomposition of \m{K} is only used as a way to construct an orthonormal basis for this subspace. 
In fact, when computational expense is a concern, the kernel can be approximated by truncating the 
eigendecomposition to \m{p<n_1+n_2} eigenvectors. 

At a more conceptual level, Theorem \ref{thm: flg compute} shows that, while we originally defined the FLG 
kernel in terms of \m{m} explicit vertex features, it is really a way of lifting a kernel between vertices 
to a kernel between graphs. 
When \m{\kappa} is a general positive semidefinite kernel, formally, \m{S_1} and \m{S_2} become operators 
on the corresponding Reproducing Kernel Hilbert Space, making the definition of the FLG kernel more 
technical. Thanks to Theorem \ref{thm: flg compute}, however, 
the actual kernel value between any two finite graphs can still be computed using finite dimensional 
linear algebra. 
}




\ignore{
Assuming \m{\Cov(\x,\x)\<=L^{-1}} and a linear feature mapping \m{\y\<=U \x},~  
\m{\Cov(\y,\y)\<=U L^{-1} U^\top}. 
Therefore, the distribution induced by \rf{eq: gm} over \m{\y} is a multivariate Gaussian with covariance 
\m{U L^{-1}U^\top}:  
\begin{equation}\label{eq: gmy}
p(\y)\propto e^{-\y^\top\!\nts  (U L^{-1}U^\top)^{-1} \ts \y/2}. 
\end{equation}
This motivates the following defintion. 
}

\ignore{
\begin{defn}
Let \m{\Gcal} be a weighted graph with vertex set \m{V\<=\cbr{1,\ldots,n}}, \m{A} be its adjacency matrix, 
and \m{B_r(j)} the neighborhood of radius \m{r} in \m{\Gcal} centered on vertex \m{j}. 
For any ordered subset \m{W} of \m{V}, let \m{[A]_{W,W}\tin\RR^{|W|\times|W|}} 
denote the submatrix of \m{A} cut out by the rows/columns indexed by elements of \m{W}. 
We say that a function \m{\phi^{\Gcal}\colon V\to\RR} is an \m{r}\df{--local invariant vertex feature} if 
\begin{compactenum}[(a)]
\item \m{\phi(j)} can be computed from just the topology of the \m{r}--neighborhood around \m{j}, i.e., 
\[\phi^\Gcal(j)=\phi([A]_{B_r(j),B_r(j)}),\]
for some appropriate matrix function \m{\phi};
\item for any permutation \m{\sigma\colon V\to V} and corresponding permutation matrix \m{P_\sigma}, 
\[\phi([P_\sigma A P_\sigma^\top]_{B_r(\sigma(j)),B_r(\sigma(j))})=\phi([A]_{B_r(j),B_r(j)}).\] 
\end{compactenum}
\end{defn}
\bigskip
}
\ignore{
\begin{prop}
Given \m{\Gcal_1} and \m{\Gcal} as in Proposition \ref{prop: subspace}, and a base kernel 
\m{\kappa(v,v')\<=\V \phi(v)\cdot \V \phi(v')} between their vertices, 
define the joint Gram matrix \m{K\tin\RR^{(n_1+n_2)\times(n_1+n_2)}}  
\smallpull
\[K_{i,j}=
\begin{cases} 
\kappa(v^{(1)}_i,v^{(1)}_j)&\Tif~~ i\<\leq n_1,~~j\<\leq n_1\\  
\kappa(v^{(1)}_i,v^{(2)}_{j-n_1})&\Tif~~ i\<\leq n_1,~~j\<>n_1\\  
\kappa(v^{(2)}_{i-n_1},v^{(1)}_{j})&\Tif~~ i\<>n_1,~~j\<\leq n_1\\  
\kappa(v^{(2)}_{i-n_1},v^{(2)}_{j-n_1})&\Tif~~ i\<>n_1,~j\<>n_1,  
\end{cases}\]
and let \m{Q D Q^\top } be the eigendecomposition of \m{K}. 
Then the vectors 
\[\xi_i=\sum_{j_1=1}^{n_1} Q_{j_1,i}\, \V \phi(v_{j_1})+\sum_{j_2=1}^{n_2} Q_{j_2,i}\, \V \phi(v'_{j_2})\quad i=1,\ldots,p\]
form an orthonormal basis for W.  
\end{prop}
}

\section{Multiscale Laplacian Graph Kernels}\label{sec: multiscale}


By a multiscale graph kernel we mean a kernel that is able to capture similarity between graphs not just based on 
the topological relationships between their individual vertices, but also the topological relationships 
between subgraphs. 
The key property of the FLG kernel that allows us to build such a kernel is that 
it can be applied recursively. In broad terms, the construction goes as follows: 
\begin{compactenum}[1.]
\item Given a graph \m{\Gcal}, divide it into a large number of small (typically overlapping) subgraphs 
and compute the FLG kernel between any two subgraphs. 
\item Each subgraph is attached to some vertex of \m{\Gcal} (for example, its center), so 
we can reinterpret the FLG kernel as a new base kernel between the vertices. 
\item We now divide \m{\Gcal} into larger subgraphs, compute the new FLG kernel between them induced from the 
new base kernel, and recurse \m{L} times. 
\end{compactenum}
Finally, to compute the actual kernel between two graphs \m{\Gcal} and \m{\Gcal'}, we follow the same process 
for \m{\Gcal'} and then compute \m{\kflg(\Gcal,\Gcal')} induced from their top level base kernels. 
The following definitions formalize this construction.   

\begin{defn}\label{def: mls}
Let \m{\Gcal} be a graph with vertex set \m{V}, and \m{\kappa} a positive semi-definite kernel on \m{V}. 
Assume that for each \m{v\tin V} we have a nested sequence of \m{L} neighborhoods
\begin{equation}\label{eq: neighborhoods}
v\tin N_1(v)\subseteq N_2(v)\subseteq \ldots \subseteq N_L(v)\subseteq V, 
\end{equation}
and for each \m{N_\ell(v)}, let \m{G_\ell(v)} be the corresponding induced subgraph of \m{\Gcal}. 
We define the   
\df{Multiscale Laplacian Subgraph Kernels (MLS kernels)}, \m{\kmls{1},\ldots,\kmls{L}\colon V\times V\to\RR} as follows:
\begin{compactenum}[1.]
\item $\mathfrak{K}_1$ is just the FLG kernel \m{\kflg^{\kappa}}
induced from the base kernel \m{\kappa} between the lowest level subgraphs: 
\[\kmls{1}(v,v')=\kflg^{\kappa}(G_1(v),G_1(v')).\]
\item For \m{\ell=2,3,\ldots,L}, the MLS kernel \m{\kmls{\ell}} is the FLG kernel induced 
from \m{\kmls{\ell-1}} between \m{G_\ell(v)} and \m{G_\ell(v')}: 
\[\kmls{\ell}(v,v')=\kflg^{\kmls{\ell-1}}(G_\ell(v),G_\ell(v')).\]
\end{compactenum}
\end{defn}

Definition \ref{def: mls} defines the MLS kernel as a kernel between different subgraphs of the same graph \m{\Gcal}. 
However, if two graphs \m{\Gcal_1} and \m{\Gcal_2} share the same base kernel, 
the MLS kernel can also be used to compare any subgraph of \m{\Gcal_1} with any subgraph of \m{\Gcal_2}. 
This is what allows us to define an \m{L\<+1}'th FLG kernel, which compares the two full graphs. 

\begin{defn}\label{def: mlg}
Let \m{\mathfrak{G}} be a collection of graphs such that all their vertices are members of an abstract vertex space 
\m{\Vcal} endowed with a symmetric positive semi-definite kernel 
\m{\kappa\colon \Vcal \times\Vcal\to\RR}. 
Assume that the MLS kernels \m{\kmls{1},\ldots,\kmls{L}} are defined as in Definition \ref{def: mls}, 
both for pairs of subgraphs within the same graph and across pairs of different graphs. 
We define the \df{Multiscale Laplacian Graph Kernel (MLG kernel)} between any two graphs 
\m{\Gcal_1,\Gcal_2\tin\mathfrak{G}} as 
\[\kmlg(\Gcal_1,\Gcal_2)=\kflg^{\kmls{L}}(\Gcal_1,\Gcal_2).\]
\end{defn}

Definition \ref{def: mlg} leaves open the question of how the neighborhoods \m{N_1(v),\ldots,N_L(v)} are 
to be defined. In the simplest case, we set \m{N_\ell(v)} to be the ball \m{B_r(v)} 
(i.e., the set of vertices at a distance at most \m{r} from \m{v}), where \m{r=r_0 \eta^{\ell-1}} 
for some \m{\eta\<>1}. The \m{\eta\<=2} case is particularly easy, 
because we can then construct the neighborhoods as follows: 
\begin{compactenum}[1.]
\item For \m{\ell\<=1}, find each \m{N_1(v)=B_{r_0}(v)} separately. 
\item For \m{\ell\<=2,3,\ldots,L}, for each \m{v\tin\Gcal} set 
\[N_\ell(v)=\bigcup_{w\in N_{\ell-1}(v)} N_{\ell-1}(w).\]
\end{compactenum}
\ignore{
Assuming that the maximal degree of any node in \m{\Gcal} is \m{D}, the first 
step can be accomplished in \m{O(n D^{r_0})} time. 
Each subsequent step involves taking \m{n} separate unions of \m{O(n)} sets of size at most \m{O(n)}.  
Assuming that \m{D^{r_0}<n^2}, this results in a total complexity of \m{O(n^3 L)}. 
}

\subsection{Computational complexity and caching} 

Definitions \ref{def: mls} and \ref{def: mlg} suggest a recurisve approach to computing the MLG kernel: 
computing \m{\kmlg(\Gcal_1,\Gcal_2)} first requires computing \m{\kmls{L}(v,v')} between all 
\m{\binom{n_1+n_2}{2}} pairs of top level subgraphs across \m{\Gcal_1} and \m{\Gcal_2}; 
each of these kernel evaluations requires computing \m{\kmls{L-1}(v,v')} between 
up to \m{O(n^2)} level \m{L\<-1} subgraphs, and so on. Following this recursion blindly would 
require up to \m{O(n^{2L+2})} kernel evaluations, which is clearly infeasible.  

The recursive strategy is wasteful because it involves evaluating the same kernel entries 
over and over again in different parts of the recursion tree. 
An alternative solution that requires only \m{O(Ln^2)} kernel evaluations would be to 
first compute \m{\kmls{1}(v,v')} for \emph{all} \m{(v,v')} pairs, then compute \m{\kmls{2}(v,v')} 
for \emph{all} \m{(v,v')} pairs, and so on. But this solution is also wasteful, because 
for low values of \m{\ell}, if \m{v} and \m{v'} are relatively distant, then they will never appear together in 
any level \m{\ell\<+1} subgraph, so \m{\kmls{\ell}(v,v')} is not needed at all. 
 The natural compromise between these two approaches is to use a recursive ``on demand'' kernel computation 
strategy, but once some \m{\kmls{\ell}(v,v')} has been computed, store it in a hash table indexed by \m{(v,v')}, 
so that \m{\kmls{\ell}(v,v')} does not need to be recomputed from scratch. 

A further source of redundancy is that in many real world graph datasets certain 
subgraphs (e.g., functional groups) recur many times over. This leads to potentially large collections 
of kernel evaluations 
\m{\cbrN{\kmls{\ell}(v,v_1'),\ldots,\kmls{\ell}(v,v_z')}} where \m{v_1'\ldots v_z'} are distinct, 
but the corresponding \m{G_\ell(v_1'),\ldots,G_\ell(v_z')} subgraphs are isomorphic 
(including the feature vectors), so the kernel values will all be the same. 
Once again, the solution is to maintain a hash table of all unique subgraphs seen so far, so that when a new 
subgraph is processed, our code can quickly determine whether it is identical to some other subgraph for 
which kernel evaluations have already been computed. Doing this process perfectly would require isomorphism 
testing, which is, of course, infeasible. In practice, a weak test that only detects a subset of 
isomorphic subgraph pairs already makes a large difference to performance.

\ignore{

The more significant expense in computing \m{\kmlg(\Gcal_1,\Gcal_2)} is computing all the MLS kernels 
in the hierarchy. 
Using the naive, purely recursive approach, computing any \m{\kmls{L}(v,v')} requires computing 
up to \m{O(n^2)} \m{\kmls{L-1}} evaluations, each of which, in turn, requires up to \m{O(n^2)} 
\m{\kmls{L-2}} evaluations, etc., leading to an overall complexity of \m{O(n^{2L+2})}, which is 
clearly infeasible. This strategy is very wasteful, because it involves computing the same kernel values 
many times over in different parts of the recursion tree. 

The natural alternative is hashing, which guarantees that for any combination of \m{v} and \m{v'}, 
\m{\kmls{\ell}(v,v')} is computed at most once. Therefore, in the worst case, we have \m{Ln^2/2} 
kernel evaluations. In reality, the number is significantly smaller because for low values of 
\m{\ell}, for distant pairs of vertices \m{\kmls{\ell}(v,v')} never needs to be evaulated. The  
cost of a single \m{\kmls{\ell}(v,v')} evaluation is dominated by computing the eigendecomposition 
of the joint kernel matrix \m{K}, which, in the can be as high as \m{O(n^3)}. Therefore, the total complexity 
of computing \m{\kmlg(\Gcal_1,\Gcal_2)} is \m{O(L n^5)}. Approximating each \m{\kmls{\ell}} in a \m{p}--dimensional 
subspace, as described in Section \ref{sec: kernelized} reduces this cost to \m{O(Lpn^4)}. 

\begin{thm}
The MLG kernel \m{\kmlg} between two graphs of \m{O(n)} vertices can be computed in \m{O(Ln^5)} time. 
Using the rank \m{p} approximation, the kernel can be approximated in \m{O(Lpn^4)} time. 
\end{thm}
}

\ignore{
We begin by defining around each vertex \m{v} of \m{\Gcal} a nested sequence of \m{L} neighborhoods
\begin{equation}\label{eq: neighborhoods}
v\subseteq N_1(v)\subseteq N_2(v)\subseteq \ldots \subseteq N_L(v).
\end{equation}
In the simplest case, \m{N_\ell(v)} is just the ball \m{B_r(v)} of radius \m{r} around \m{v} 
(i.e., the set of vertices at a distance at most \m{r} from \m{v}), 
where \m{r} might grow with some power of \m{\ell}, e.g., \m{r\<=2^{\ell+1}}. 
We let \m{G_\ell(v)} 
be the subgraph of \m{\Gcal} induced by \m{N_\ell(v)}, and 
\m{L_\ell(v)} the Laplacian of this subgraph. 

\begin{defn}\label{defn: mls}
Given a collection of graphs, a base kernel \m{\kappa} between their vertices, 
and an appropriate nested sequence of neighborhoods \rf{eq: neighborhoods} for each vertex of each 
graph, the \df{Multiscale Laplacian Subgraph Kernel (MLS kernel)} \m{\kmls{\ell}} 
between two vertices \m{v} and \m{v'} 
(which might or might not be in the same graph) is defined as follows:
\begin{compactenum}[1.]
\item At level \m{\ell\<=1} the MLS kernel is just the FLG kernel \m{\kflg{\kappa}} 
based on \m{\kappa} between the lowest level subgraphs \m{G_1(v)} and \m{G_1(v')}: 
\[\kmls{1}(v,v')=\kflg{\kappa}(G_1(v),G_1(v')).\]
\item At levels \m{\ell=2,3,\ldots,L}, the MLS kernel is the FLG kernel based on \m{\kmls{\ell-1}} 
between the higher level subgraphs \m{G_\ell(v)} and \m{G_\ell(v')}: 
\[\kmls{\ell}(v,v')=\kflg{\kmls{\ell-1}}(G_\ell(v),G_\ell(v')).\]
\end{compactenum}
\end{defn}

\begin{defn}
Given a pair of graphs \m{(\Gcal_1,\Gcal_2)} 
and everything else as in Definition \ref{defn: mls}, the \df{Multiscale Laplacian Graph Kernel (MLG kernel)} 
between them is defined as 
\[\kmlg(\Gcal_1,\Gcal_2)=\kflg{\kmls{L}}(\Gcal_1,\Gcal_2).\]
\end{defn}

\subsection{Computational complexity}

In the simplest case, where the \m{N_\ell(v)} neigborhoods are balls of radius \m{r=r_0 2^{\ell}} 
around each vertex \m{v}, they can be constructed with using the following recursive procedure:
\begin{compactenum}[1.]
\item For \m{\ell\<=1}, find each \m{N_1(v)=B_{r_0}(v)} separately. 
\item For \m{\ell\<=2,3,\ldots,L}, for each \m{v\tin\Gcal} set 
\[N_\ell(v)=\bigcup_{w\in N_{\ell-1}(v)} N_{\ell-1}(w).\]
\end{compactenum}
Assuming that the maximal degree of any node in \m{\Gcal} is \m{D}, the first 
step can be accomplished in \m{O(n D^r_0)} time. 
Each subsequent step involves taking \m{n} separate unions of \m{O(n)} sets of size at most \m{O(n)}.  
Assuming that \m{D^{r_0}<n^2}, this results in a total complexity of \m{O(n^3 L)}. 
Computing the Laplacian of each induced subgraph is 
done in time \m{O(n^2)}, so all the Laplacians can also be computed in time \m{O(n^3L)}.  
}

\section{Linearized Kernels and Low Rank Approximation}\label{sec: lowrank}

Even with caching, MLS and MLG kernels can be expensive to compute. 
The main reason for this is that they involve expressions like \rf{eq: LG}, 
where \m{S_1} and \m{S_2} are initially given in different bases. 
To find a common basis for the two matrices via the method of Proposition \ref{prop: Smx} requires 
a potentially large number of lower level kernel evaluations, which require even lower level 
kernel evaluations, and so on. 
Unfortunately, this process has to be repeated anew for each \m{\cbrN{\Gcal_1,\Gcal_2}} pair, 
because in all \m{\kmls{\ell}(v,v')} evaluations where \m{v} is from one graph and \m{v'} is from the 
other, the common basis will involve both graphs. 
Consequently, the cost of the basis computations cannot be amortized into a per-graph precomputation stage.  

In the previous section we saw that computing the MLG kernel between two graphs may involve \m{O(Ln^2)} 
kernel evaluations.  
At the top levels of the hierarchy each \m{G_\ell(v)} might have \m{\Theta(n)} vertices, so the 
cost of a single FLG kernel evaluation can be as high as \m{O(n^3)}. 
Somewhat pessimistically, this means that the overall 
cost of computing \m{\kflg(\Gcal_1,\Gcal_2)} is \m{O(Ln^5)}. 
Given a dataset of \m{M} graphs, computing their Gram matrix requires repeating this for all 
\m{\cbrN{\Gcal_1,\Gcal_2}} pairs, giving \m{O(LM^2n^5)}, which is even more problematic. 

The solution that we propose is to compute for each level \m{\ell=1,2,\ldots,L\<+1} 
a single joint basis for \emph{all} subgraphs at the given level across \emph{all} graphs \m{\sseq{\Gcal}{M}}.
For concreteness, we go back to the definition of the FLG kernel. 

\begin{defn}\label{def: W}
Let \m{\mathfrak{G}=\cbrN{\sseq{\Gcal}{M}}} be a collection of graphs, \m{\sseq{V}{M}} 
their vertex sets, and assume that \m{\sseq{V}{M}\subseteq \Vcal} for some general vertex space \m{\Vcal}.
Further, assume that \m{\kappa\colon\Vcal\times\Vcal\to\RR} is a positive semi-definite kernel on \m{\Vcal}, 
\m{\Hcal_\kappa} is its Reproducing Kernel Hilbert Space, and \m{\phi\colon \Vcal\to\Hcal_\kappa} 
is the corresponding feature map satisfying \m{\kappa(v,v')=\inp{\phi(v),\phi(v')}} for any \m{v,v'\tin\Vcal}. 
The \df{joint vertex feature space} of \m{\cbrN{\sseq{\Gcal}{M}}} is then 
\[W_{\mathfrak{G}}=\Tspan\sqbbigg{ \bigcup_{i=1}^{M}\bigcup_{v\in V_i} \cbr{\phi(v)}}.\]
\end{defn}

\m{W_{\mathfrak{G}}} is just the generalization of the \m{W} space defined in Proposition \ref{prop: subspace} 
from two graphs to \m{M}. In particular, for any \m{\cbr{\Gcal,\Gcal'}} pair (with \m{\Gcal,\Gcal'\tin\mathfrak{G}}) 
the corresponding \m{W} space will be a subspace of \m{W_{\mathfrak{G}}}. 
The following generalization of Propositions \ref{prop: subspace} and \ref{prop: Smx} is then immediate. 

\begin{prop}\label{prop: Sjoint}
Let \m{N\<=\sum_{i=1}^{M} \abs{V_i}},~ \m{\wbar V=(\sseq{\wbar v}{N})} 
be the concatination of the vertex sets \m{\sseq{V}{M}}, and \m{K} the corresponding Gram matrix 
\begin{equation}\label{eq: jointGram}
K_{i,j}=\kappa(\wbar v_i,\wbar v_j)=\inp{\phi(\wbar v_i), \phi (\wbar v_j)}.
\end{equation}
Let \m{\sseq{\V u}{P}} be a maximal orthonormal set of non-zero eigenvalue eigenvectors of \m{K} 
with corresponding eigenvalues \m{\sseq{\lambda}{P}}. 
Then the vectors 
\begin{equation*}
\xi_i=\ovr{\sqrt{\lambda_i}} \sum_{\ell=1}^{N} [\V u_i]_\ell\: \phi(\wbar v_\ell) \qquad i=1,\ldots,P
\end{equation*}
form an orthonormal basis for \m{W_{\mathfrak{G}}}. 
Moreover, defining \m{Q=[\lambda_1^{1/2}\V u_1, \ldots, \lambda_p^{1/2} \V u_p]\tin\RR^{P\times P}},  
and setting \m{Q_1} to be the submatrix of \m{Q} composed of its first \m{\absN{\!V_1\!}} rows; 
\m{Q_2} be the submatrix composed of the next \m{\abs{\!V_2\!}} rows, and so on,  
for any \m{\Gcal_i,\Gcal_j\tin\mathfrak{G}}, the generalized FLG kernel induced from \m{\kappa} 
(Definition \ref{def: gFLG}) can be expressed as 
\smallpull 
\begin{equation}\label{eq: Sjoint1}
\kflg(\Gcal_i,\Gcal_j)=
\fr{\absbig{\nts \brbig{\ovr{2}\wbar S_i^{-1}\<+\ovr{2} \wbar S_j^{-1}}^{-1}\nts }^{1/2}}
{\absN{\nts \wbar S_i\nts }^{1/4}\; \absN{\nts \wbar S_j\nts }^{1/4}}\:, \smallpull 
\end{equation}
where 
\m{\wbar S_i=Q_i^\top L_i^{-1} Q_i +\gamma I} and 
\m{\wbar S_j=Q_j^\top L_j^{-1} Q_j +\gamma I}.
\end{prop}

The significance of Proposition \ref{prop: Sjoint} is that \m{\sseq{S}{M}} are now fixed matrices that 
do not need to be recomputed for each kernel evaluation. Once we have constructed the joint basis \m{\cbr{\sseq{\xi}{P}}}, 
the \m{S_i} matrix of each graph \m{\Gcal_i} can be computed independently, as a precomputation step, and 
individual kernel evaluations reduce to just plugging them into \rf{eq: Sjoint1}. 
At a conceptual level, what Proposition \ref{prop: Sjoint} does it to linearize the kernel \m{\kappa} 
by projecting everything down to \m{W_{\Gfrak}}. 
In particular, it replaces the \m{\cbr{\phi(\wbar{v}_i)}} RKHS vectors with explicit finite dimensional feature 
vectors given by the corresponding rows of \m{Q}, just like we had in the ``unkernelized'' FLG kernel of 
Definition \ref{def: FLG}. 

For our multiscale kernels this is particularly important, because linearizing not just \m{\kflg^{\kappa}}, 
but also \m{\kflg^{\kmls{1}},\kflg^{\kmls{2}},\ldots}, allows us to compute the MLG kernel 
level by level, without recursion. 
After linearizing the base kernel \m{\kappa}, we can attach explicit, finite dimensional vectors to each vertex of 
each graph. Then we compute compute \m{\kflg^{\kmls{1}}} between all pairs of lowest level subgraphs, and 
linearizing this kernel as well, each vertex effectively just gets an updated feature vector. 
Then we repeat the process for \m{\kflg^{\kmls{2}}\ldots\kflg^{\kmls{L}}}, and finally we 
compute the MLG kernel \m{\kmlg{(\Gcal_1,\Gcal_2)}}. 

\subsection{Randomized low rank approximation}

The difficulty in the above approach of course is that at each level \rf{eq: jointGram} is a Gram matrix between 
\emph{all} vertices of \emph{all} graphs, so storing it is already very costly, let along 
computing its eigendecomposition. 
Morever, \m{P=\dim(W_{\Gfrak})} is also very large, so managing the \m{\sseq{\wbar{S}}{M}} matrices 
(each of which is of size \m{P\<\times P}) becomes infeasible.  
The natural alternative is to replace \m{W_{\Gfrak}} by a smaller, \emph{approximate} joint features space,
defined as follows. 

\begin{defn}\label{def: randomized}
Let \m{\Gfrak, \kappa, \Hcal_\kappa} and \m{\phi} be defined as in Definition \ref{def: W}. 
Let \m{\tilde V=(\sseq{\tilde v}{\tilde N})} be \m{\tilde N\<\ll N} vertices sampled from the 
joint vertex set \m{\wbar V=(\sseq{\wbar v}{N})}. Then the corresponding \df{subsampled vertex feature space} is 
\[\tilde W_{\mathfrak{G}}=\Tspan\setofN{\phi(\tilde v)}{\tilde v\tin \tilde V}.\]
\end{defn}

Similarly to before, we construct an orthonormal basis \m{\cbrN{\sseq{\xi}{P}}} for 
\m{\tilde W} by forming the (now much smaller) Gram matrix \m{\tilde K_{i,j}=\kappa(\tilde v_i,\tilde v_j)}, 
computing its eigenvalues and eigenvectors, and setting 
\m{\xi_i=\ovr{\sqrt{\lambda_i}} \sum_{\ell=1}^{\tilde N} [\V u_i]_\ell\: \phi(\tilde v_\ell)}. 
The resulting approximate FLG kernel is  
\begin{equation}\label{eq: Nystrom1}
\kflg(\Gcal_i,\Gcal_j)=
\fr{\absbig{\nts \brbig{\ovr{2}\tilde S_i^{-1}\<+\ovr{2} \tilde S_j^{-1}}^{-1}\nts }^{1/2}}
{\absN{\nts \tilde S_i\nts }^{1/4}\; \absN{\nts \tilde S_j\nts }^{1/4}},   
\end{equation}
where 
\m{\tilde S_i=\tilde Q_i^\top L_i^{-1} \tilde Q_i +\gamma I} and 
\m{\tilde S_j=\tilde Q_j^\top L_j^{-1} \tilde Q_j +\gamma I} are 
the projections of \m{\wbar S_i} and \m{\wbar S_j} to \m{\tilde W_{\Gfrak}}. 
We introduce a further layer of approximation by restricting \m{\tilde W_{\Gfrak}} 
to be the space spanned by the first \m{\tilde P\<<P} basis vectors  
(ordered by descending eigenvalue), effectively doing kernel PCA on \m{\cbr{\phi(\tilde v)}_{\tilde v\in\tilde V}}, 
equivalently, a low rank approximation of \m{\tilde K}. 

Assuming that \m{v^g_j} is the \m{j}'th vertex of \m{\Gcal_g}, 
in constrast to Proposition \ref{prop: Smx}, 
now the \m{j}'th row of \m{\tilde Q_s} consists of the coordinates of 
the \emph{projection} of \m{\phi(v^g_j)} onto \m{\tilde W_{\Gfrak}}, i.e., 
\begin{multline*}
[\tilde Q^g]_{j,i}=\ovr{\sqrt{\lambda_i}} \sum_{\ell=1}^{\tilde N} [\V u_i]_\ell\:\inp{\phi(v^g_j),\phi(\tilde v_N)}=\\
\ovr{\sqrt{\lambda_i}} \sum_{\ell=1}^{\tilde N} [\V u_i]_\ell\:\kappa(v^g_j,\tilde v_N).
\end{multline*}

The above procedure is similar to the popular Nystr\"om approximation for kernel matrices 
\cite{Williams2001,Drineas2005}, 
except that in our case the ultimate goal is not to approximate the Gram matrix \rf{eq: jointGram} itself, 
but the \m{\sseq{S}{M}} matrices used to form the FLG kernel. 
In practice, we found that the eigenvalues of \m{K} usually drop off very rapidly, 
suggesting that \m{W} can be safely approximated by a surprisingly small dimensional subspace 
(\m{\tilde P\sim 10}), and correspondingly the sample size \m{\tilde N} can be kept quite small 
as well (on the order of \m{100}). The combination of these two factors makes computing the entire stack 
of kernels very fast, reducing the complexity of computing the Gram matrix for a dataset of \m{M} 
graphs of \m{\theta(n)} vertices each to \m{O(ML{\tilde N}^2 \tilde P^3+ML{\tilde N}^3+M^2 \tilde P^3)}. 
As an example, for the ENZYMES dataset, comprised of 600 graphs, the FLG kernel between 
all pairs of graphs can be computed in about 2 minutes on a 16 core machine. 

Note that Definition \ref{def: randomized} is noncommittal to the sampling distribution used to select 
\m{(\sseq{\tilde v}{\tilde N})}: 
in our experiments we used uniform sampling without replacement.  
Also note that regardless of the approximations, \m{\sseq{S}{M}} matrices are always positive definite, 
and this fact alone, by the definition of the Bhattacharyya kernel, 
guarantees that the resulting FLG, MLS and MLG kernels are positive semi-definite kernels. 
For a high level pseudocode of the resulting algorithm, see the Supplementary Materials.  

\section{Experiments}\label{sec: experiments}
We compared the efficacy of the MLG kernel with some of the top performing graph kernels from the literature: 
the Weisfeiler--Lehman Kernel, the Weisfeiler--Lehman Edge Kernel \cite{ShervEtal09}, 
the Shortest Path Kernel \cite{BorKri05}, 
the Graphlet Kernel \cite{ShervEtal09},
and the \m{p}-random Walk Kernel \cite{VisBorKonSch10}, 
 on standard benchmark datasets(Table \ref{tbl: datasets}).
\begin{table*}[]
\small
\centering
\vspace{-10pt}
\caption{\label{tbl: res1}Classification Results (Accuracy \m{\pm} Standard Deviation)}
\begin{adjustbox}{width=1\textwidth}
\begin{tabular}{ l c c c c c c}
\hline 
Method & MUTAG\cite{DebnatEtAl} & PTC\cite{ToivonenEtAl}  & ENZYMES\cite{BorOngSchVisetal05} & PROTEINS\cite{BorOngSchVisetal05} & NCI1\cite{WaleEtAl} & NCI109\cite{WaleEtAl} \\
\hline
WL        & \m{84.50(\pm 2.16)} & \m{59.97(\pm 1.60)} & \m{53.75(\pm 1.37)} 
& \m{75.49(\pm 0.57)} & \boldmath{\m{84.76(\pm 0.32)}}& \boldmath{\m{85.12(\pm 0.29)}}
\\
WL-Edge   & \m{82.94(\pm 2.33)} & \m{60.18(\pm 2.19)} & \m{52.00(\pm 0.72)} 
& \m{74.78(\pm 0.59)} & \boldmath{\m{84.65(\pm 0.25)}}& \boldmath{\m{85.32(\pm 0.34)}}
\\
SP                & \m{85.50(\pm 2.50)} & \m{59.53(\pm 1.71)} & \m{42.31(\pm 1.37)} 
& \m{75.61(\pm 0.45)}& \m{73.61(\pm 0.36)}& \m{73.23(\pm 0.26)}
\\
Graphlet 	                & \m{82.44(\pm 1.29)} & \m{55.88(\pm 0.31)} & \m{10.95(\pm 0.69)}
& \m{71.63(\pm 0.33)}& \m{62.40(\pm 0.27)}& \m{62.35(\pm 0.28)}
\\
\m{p}--RW        & \m{80.33(\pm 1.35)} & \m{59.85(\pm 0.95)} & \m{28.17(\pm 0.76)} 	
& \m{71.67(\pm 0.78)}  & TIMED OUT(\textgreater 24hrs) & TIMED OUT(\textgreater 24hrs)
\\
MLG		& \boldmath{ \m{87.94(\pm 1.61)}} &	\boldmath{\m{63.26(\pm 1.48)}} & \boldmath{\m{61.81(\pm 0.99)}}
& \boldmath{\m{76.34(\pm 0.72)}} & \m{81.75(\pm 0.24)} & \m{81.31(\pm 0.22)} \\
\hline	 
\end{tabular}
\end{adjustbox}
\end{table*}

\begin{table}[]
\centering
\caption{\label{tbl: datasets}Summary of the datasets used in our experiments}
\begin{adjustbox}{width=1\columnwidth}
\begin{tabular}{l l l l l l l}
\hline 
Dataset & Size   & Labels& Nodes & Edges & Diameter & Classes   \\
\hline                                             
MUTAG    & 188   & 7     & 17.9 & 39.6  & 8.2      & 2 (125 vs 63) \\
PTC      & 344   & 19    & 25.6 & 51.9  & 8.9      & 2 (192 vs 152)  \\
ENZYMES  & 600   & 3     & 32.6 & 124.3 & 10.9     & 6 (100 each)\\   
PROTEINS & 1113  & 3     & 39.1 & 145.6 & 11.6     & 2 (663 vs 450)\\
NCI1     & 4110  & 37    & 29.9 & 64.6  & 13.3     & 2 (2057 vs 2053)\\
NCI109   & 4127  & 38    & 29.7 & 64.3  & 13.1     & 2 (2079 vs 2048)\\
\hline
\end{tabular}
\vspace{-10pt}
\end{adjustbox}
\end{table}

We perform classification using a binary C-SVM solver \cite{ChangLin11} to test our kernel method. 
We tuned the SVM slack parameter through 10-fold cross-validation using 9 folds for training and 1 for testing, 
repeated 10 times. All experiments were done on a 16 core Intel E5-2670 @ 2.6GHz processor with 32 GB of memory.
Our prediction accuracy and standard deviations are shown in Table \ref{tbl: res1} and runtimes in Table \ref{tbl: res2}. 

The parameters for each kernel were chosen as follows:
for the Weisfeiler--Lehman kernels, the height parameter $h$ is chosen from $\{1,2,..., 5\}$,
the random walk size $p$ for the $p$-random walk kernel was chosen from $\{1, 2,..., 5\}$, for the Graphlets kernel the graphlet size $n$ was chosen from $\{3, 4, 5\}$
as outlined in \cite{ShervEtal09}.
For the parameters of the MLG kernel: we chose $\eta$ and $\gamma$ from $\{0.01, 0.1, 1\}$, radius size $n$ from $\{1, 2, 3, 4\}$, number of levels $l$ from $\{1, 2,3 , 4\}$.
We used the given discrete node labels to create a one-hot binary feature vector for each node and used the dot product between nodes' binary feature vector labels as the
base kernel for the MLG kernel.

\begin{table}[]
\centering
\caption{\label{tbl: res2}Runtime of MLG on different Datasets }
\begin{tabular}{l c c}
\hline 
Dataset     & Wall clock time & CPU time    \\
\hline
MUTAG       & 0min 0.86s      & 0min 5.7s \\
PTC         & 1min 11.18s     & 9min 9.5s \\
ENZYMES     & 0min 36.65s     & 4min 41.2s  \\   
PROTEINS    & 3min 19.8s      & 48min 23.0s  \\   
NCI1        & 5min 36.3s      & 84min 4.8s  \\   
NCI109      & 5min 42.6s      & 84min 35.9s  \\   
\hline
\end{tabular}\vspace{-10pt}
\end{table}

We achieve the highest prediction accuracy for all datasets except NCI1 and NCI109, 
where it performs better than all non-Weisfeiler Lehman kernels. 
Across all datasets, we found the optimal number of levels to be $2$ or $3$ and likewise for the radius size. 
As can be seen from the average number of nodes and average diameter values in Table \ref{tbl: datasets}, the graphs in each dataset are small enough that a $2$ or $3$ level deep MLG kernel is sufficient to effectively characterize the similarity between graphs. The optimal $\eta$ and $\gamma$ values were either $0.01$ or $0.1$ in all cases. In general, these two parameters can be set through cross validation over a small set of values. For two graphs $G$ and $\hat{G}$, that are reasonably similar with only slight differences(ex: $\hat{G}$ is similar to $G$ in degree distribution, connectivity, etc), increasing
the $\eta$ and/or $\gamma$ value will have the effect of artificially increasing the value of $k_{FLG} (G, \hat{G})$, smoothing out their differences. This sort of smoothing
is not desirable for all pairs of graphs, so typically the optimal $\eta$ and $\gamma$ values will be small, often between $0.01$ and $1$.

\section{Conclusions}

In this paper we have proposed two new graph kernels: 
(1) The FLG kernel, which is a very simple single level kernel that combines information attached 
to the vertices with the graph Laplacian;  
(2) The MLG kernel, which is a multilevel, recursively defined kernel that captures topological 
relationships between not just individual vertices, but also subgraphs. 
Clearly, designing kernels that can optimally take into account the multiscale structure of actual chemical 
compounds is a challenging task that will require further work and domain knowledge. 
However, it is encouraging that even just ``straight out of the box'', tuning only one or two parameters, such as 
the number of levels, the MLG kernel performed on par with, or even slightly better than the other well known 
kernels in the literature. 
Beyond just graphs, the general idea of multiscale kernels is of interest for other types of data as well 
(such as images) that have multiresolution structure, and the way that the MLG kernel chains together 
local spectral analysis at multiple scales is potentially applicable to these domains as well, which will be
the subject of further research.

\section*{Acknowledgements}\label{sec: acknowledgements}
This work was completed in part with computing resources provided by the University of Chicago Research
Computing Center.

\clearpage
\bibliographystyle{icml2016}
\bibliography{MSG}

\begin{thebibliography}{22}
\providecommand{\natexlab}[1]{#1}
\providecommand{\url}[1]{\texttt{#1}}
\expandafter\ifx\csname urlstyle\endcsname\relax
  \providecommand{\doi}[1]{doi: #1}\else
  \providecommand{\doi}{doi: \begingroup \urlstyle{rm}\Url}\fi

\bibitem[Alexa et~al.(2009)Alexa, Kazhdan, and Guibas]{Alexa2009}
Alexa, Marc, Kazhdan, Michael, and Guibas, Leonidas.
\newblock {A Concise and Provably Informative Multi-Scale Signature Based on
  Heat Diffusion}.
\newblock In \emph{Processing of Eurographics Symposium on Geometry
  Processing}, volume~28, 2009.

\bibitem[Borgwardt et~al.(2005)Borgwardt, Ong, Sch\"{o}nauer, Vishwanathan,
  Smola, and Kriegel]{BorOngSchVisetal05}
Borgwardt, K.~M., Ong, C.~S., Sch\"{o}nauer, S., Vishwanathan, S.~V.~N., Smola,
  A.~J., and Kriegel, H.-P.
\newblock Protein function prediction via graph kernels.
\newblock In \emph{Proceedings of Intelligent Systems in Molecular Biology
  (ISMB)}, Detroit, USA, 2005.

\bibitem[Borgwardt \& Kriegel(2005)Borgwardt and Kriegel]{BorKri05}
Borgwardt, Karsten~M. and Kriegel, Hans~Peter.
\newblock Shortest-path kernels on graphs.
\newblock In \emph{Proceedings of the 5th IEEE International Conference on Data
  Mining(ICDM) 2005), 27-30 November 2005, Houston, Texas, USA}, pp.\  74--81,
  2005.

\bibitem[Chang \& Lin(2011)Chang and Lin]{ChangLin11}
Chang, Chih-Chung and Lin, Chih-Jen.
\newblock Libsvm: A library for support vector machines.
\newblock \emph{ACM Transactions on Intelligent Systems and Technology}, 3,
  2011.

\bibitem[Debnat et~al.(1991)Debnat, de~Compadre, Debnath, j.~Shusterman, and
  Hansch]{DebnatEtAl}
Debnat, A.K., de~Compadre, R. L.~Lopez, Debnath, G., j.~Shusterman, A., and
  Hansch, C.
\newblock Structure-activity relationship of mutagenic aromatic and
  heteroaromatic nitro compounds. correlation with molecular orbital energies
  and hydrophobicity.
\newblock \emph{J Med Chem}, 34:\penalty0 786--97, 1991.

\bibitem[Drineas \& Mahoney(2005)Drineas and Mahoney]{Drineas2005}
Drineas, Petros and Mahoney, Michael~W.
\newblock On the {N}ystr{\"o}m method for approximating a {G}ram matrix for
  improved kernel-based learning.
\newblock \emph{Journal of Machine Learning Research}, 6:\penalty0 2153--2175,
  2005.

\bibitem[Feragen et~al.(2013)Feragen, Kasenburg, Petersen, de~Bruijne, and
  Borgwardt]{FerEtal13}
Feragen, Aasa, Kasenburg, Niklas, Petersen, Jens, de~Bruijne, Marleen, and
  Borgwardt, Karsten~M.
\newblock Scalable kernels for graphs with continuous attributes.
\newblock In \emph{Advances in Neural Information Processing Systems 26: 27th
  Annual Conference on Neural Information Processing Systems 2013. Proceedings
  of a meeting held December 5-8, 2013, Lake Tahoe, Nevada, United States.},
  pp.\  216--224, 2013.

\bibitem[G\"artner(2002)]{Gaertner02}
G\"artner, T.
\newblock Exponential and geometric kernels for graphs.
\newblock In \emph{NIPS*02 workshop on unreal data}, volume Principles of
  modeling nonvectorial data, 2002.

\bibitem[H.Toivonen et~al.(2003)H.Toivonen, Srinivasan, King, Kramer, and
  Helma]{ToivonenEtAl}
H.Toivonen, Srinivasan, A., King, R.~D., Kramer, S., and Helma, C.
\newblock Statistical evaluation of the predictive toxicology challenge.
\newblock \emph{Bioinformatics}, pp.\  1183--1193, 2003.

\bibitem[Inokuchi et~al.(2003)Inokuchi, Washio, and Motoda]{InoWasMot03}
Inokuchi, Akihiro, Washio, Takashi, and Motoda, Hiroshi.
\newblock Complete mining of frequent patterns from graphs: Mining graph data.
\newblock \emph{Machine Learning}, 50\penalty0 (3):\penalty0 321--354, 2003.

\bibitem[Jebara \& Kondor(2003)Jebara and Kondor]{JebKon03}
Jebara, Tony and Kondor, Risi.
\newblock Bhattacharyya and expected likelihood kernels.
\newblock In Sch\"olkopf, B. and Warmuth, M. (eds.), \emph{Proceedings of the
  Annual Conference on Computational Learning Theory and Kernels Workshop
  ({COLT/KW})}, number 2777 in Lecture Notes in Computer Science, pp.\  57--71,
  Heidelberg, Germany, 2003. Springer-Verlag.

\bibitem[Kondor \& Borgwardt(2008)Kondor and Borgwardt]{KonBor08}
Kondor, Risi and Borgwardt, Karsten.
\newblock The skew spectrum of graphs.
\newblock In \emph{Proceedings of the International Conference on Machine
  Learning (ICML)}, pp.\  496--503. ACM, 2008.

\bibitem[Kondor \& Jebara(2003)Kondor and Jebara]{KonJeb03}
Kondor, Risi and Jebara, Tony.
\newblock A kernel between sets of vectors.
\newblock In \emph{Proceedings of the International Conference on Machine
  Learning (ICML)}, 2003.
\newblock (Best student paper award. Google Scholar citations: 137).

\bibitem[Kubinyi(2003)]{Kubinyi03}
Kubinyi, H.
\newblock Drug research: myths, hype and reality.
\newblock \emph{Nature Reviews: Drug Discovery}, 2\penalty0 (8):\penalty0
  665--668, August 2003.

\bibitem[Mika et~al.(1999)Mika, {Sch\"olkopf}, Smola, {M\"uller}, Scholz, and
  {R\"atsch}]{MikSchSmoMuletal99}
Mika, S., {Sch\"olkopf}, B., Smola, A.~J., {M\"uller}, K.-R., Scholz, Matthias,
  and {R\"atsch}, G.
\newblock Kernel {PCA} and de-noising in feature spaces.
\newblock In Kearns, M.~S., Solla, S.~A., and Cohn, D.~A. (eds.),
  \emph{Advances in Neural Information Processing Systems 11}, pp.\  536--542.
  {MIT} Press, 1999.

\bibitem[Mikkonen(2007)]{MikkonenGraphicValues}
Mikkonen, T.
\newblock The ring of graph invariants --- graphic values.
\newblock 2007.

\bibitem[Sch{\"o}lkopf \& Smola(2002)Sch{\"o}lkopf and Smola]{SchSmo02}
Sch{\"o}lkopf, Bernhard and Smola, Alexander~J.
\newblock \emph{Learning with Kernels}.
\newblock {MIT} Press, 2002.

\bibitem[Shervashidze et~al.(2009)Shervashidze, Vishwanathan, Petri, Mehlhorn,
  and Borgwardt]{ShervEtal09}
Shervashidze, Nino, Vishwanathan, S.~V.~N., Petri, Tobias, Mehlhorn, Kurt, and
  Borgwardt, Karsten~M.
\newblock Efficient graphlet kernels for large graph comparison.
\newblock In \emph{Proceedings of the Twelfth International Conference on
  Artificial Intelligence and Statistics, {AISTATS} 2009, Clearwater Beach,
  Florida, USA, April 16-18, 2009}, pp.\  488--495, 2009.

\bibitem[Shervashidze et~al.(2011)Shervashidze, Schweitzer, van Leeuwen,
  Mehlhorn, and Borgwardt]{ShervEtal11}
Shervashidze, Nino, Schweitzer, Pascal, van Leeuwen, Erik~Jan, Mehlhorn, Kurt,
  and Borgwardt, Karsten~M.
\newblock Weisfeiler-lehman graph kernels.
\newblock \emph{jmlr}, 12:\penalty0 2539--2561, November 2011.

\bibitem[Vishwanathan et~al.(2010)Vishwanathan, Borgwardt, Kondor, and
  Schraudolph]{VisBorKonSch10}
Vishwanathan, S.~V.~N., Borgwardt, Karsten, Kondor, Risi, and Schraudolph,
  Nicol.
\newblock On graph kernels.
\newblock \emph{Journal of Machine Learning Research (JMLR)}, 11, 2010.

\bibitem[Wale et~al.(2008)Wale, Watson, and Karypis]{WaleEtAl}
Wale, N., Watson, I.~A., and Karypis, G.
\newblock Comparison of descriptor spaces for chemical compound retrieval and
  classification.
\newblock \emph{Knowledge and Information Systems}, pp.\  347--375, 2008.

\bibitem[Williams \& Seeger(2001)Williams and Seeger]{Williams2001}
Williams, Christopher K.~I. and Seeger, Mattias.
\newblock Using the {N}ystr\"{o}m method to speed up kernel machines.
\newblock In \emph{Advances in Neural Information Processing Systems (NIPS)},
  2001.

\end{thebibliography}
\end{document}